\documentclass[10pt,twoside]{article}

\usepackage[english]{babel}
\usepackage[utf8]{inputenc}
\usepackage{amsmath}
\usepackage{amsfonts}
\usepackage{mathtools}
\usepackage{amssymb}
\usepackage{amsthm}
\usepackage{hyperref}
\usepackage{graphicx}
\usepackage{epstopdf}
\usepackage{todonotes}

\def\N			{\mathbb N}

\def\R			{\mathbb R}

\def\st{\, {\bf :} \,}
\def\pd{{\bf PD}}

\def\cH	        {\mathcal H}
\def\cL	        {\mathcal L}
\def\cR	        {\mathcal R}
\def\cS	        {\mathcal S}

\def\elltwo     {\ell_2(\mathbb{N})}

\newtheorem{theorem}{Theorem}[section]
\newtheorem{lemma}[theorem]{Lemma}
\newtheorem{corollary}[theorem]{Corollary}

\theoremstyle{definition}
\newtheorem{definition}[theorem]{Definition}

\begin{document}

%#########################################################################
%### Page design
%#########################################################################

\pagestyle{myheadings}

%#########################################################################
%### Title page
%#########################################################################

\title{Irregular Sampling of High-Dimensional Functions in Reproducing Kernel Hilbert Spaces}

\date{\today}

\author{\large
Armin Iske\footnote{Universit\"at Hamburg, Department of Mathematics, {\tt armin.iske@uni-hamburg.de}} \qquad
Lennart Ohlsen\footnote{Universit\"at Hamburg, Dept of Mathematics, {\tt lennart.ohlsen@studium.uni-hamburg.de}}}

\markboth{\footnotesize \rm \hfill A.~ISKE AND L.~OHLSEN\hfill}
{\footnotesize \rm \hfill IRREGULAR SAMPLING IN REPRODUCING KERNEL HILBERT SPACES\hfill}

\maketitle
\thispagestyle{plain}

%%%%%%%%%%%%%%%%%%%%%%%%%%%%%%%%%%%
% A B S T R A C T
%%%%%%%%%%%%%%%%%%%%%%%%%%%%%%%%%%%

\begin{abstract}
We develop sampling formulas for high-dimensional functions in reproducing kernel Hilbert spaces,
where we rely on irregular samples that are taken at {\em determining} sequences of data points. 
We place particular emphasis on sampling formulas for
tensor product kernels, where we show that determining irregular samples in lower dimensions
can be composed to obtain a tensor of determining irregular samples in higher dimensions.
This in turn reduces the computational complexity of sampling formulas for high-dimensional functions quite significantly.
\end{abstract}

%%%%%%%%%%%%%%%%%%%%%%%%%%%%%%%%%%%%%%%%%%%%%%%%%%%%%%%%%%%%%%%%
% I N T R O D U C T I O N
%%%%%%%%%%%%%%%%%%%%%%%%%%%%%%%%%%%%%%%%%%%%%%%%%%%%%%%%%%%%%%%%
%------------------------------------------------------------------------------------------
\section{Introduction}
%------------------------------------------------------------------------------------------
Reproducing kernel Hilbert spaces (RKHS) play an increa\-singly important role in mathematical data ana\-lysis.
The long history of reproducing kernels is dating back to the seminal work~\cite{Aronszajn1950} of Aronszajn (in 1950), 
followed by a wide range of fundamental contributions from approximation theory~\cite{Micchelli2006,Zhou2003},
computational harmonic analysis~\cite{Bartolucci2023},
learning theory~\cite{Cucker2007,Schoelkopf2002}, support vector machines~\cite{Steinwart2008}, 
numerical analysis~\cite{Schaback2006}, probability, measure theory, and statistics~\cite{Berlinet2004}.

More recently, 
kernel regression for high-dimensional data fitting through dimensionality reduction has gained enormous popu\-larity
in statistical data analysis~\cite{Fukumizu2009}, machine lear\-ning~\cite{EIT2023} and manifold learning~\cite{Guillemard2017}.
Other relevant applications include particle-based fluid flow simulations,
requiring adaptive selections of sample points~\cite{AEI2023,AI2025,BIP2001}.
In that case, the construction and ana\-lysis of sampling formulas in RKHS are of vital importance
to approximate multivariate functions. 

Sampling theorems have a long history, too, dating back to the celebrated work~\cite{Shannon1949} of Shannon (in 1949).
Shannon's sampling theorem represents band-limited signals on uniform grids, i.e., from regular samples.
For a comprehensive account to the history of Shannon's theorem, we refer to~\cite{Unser2000}.

\medskip
In this work, we analyse sampling formulas of the form
\begin{equation}
\label{eq:st}
        f=\sum_{k\in\N} f(x_k) L_k,
\end{equation}
for functions $f$ in Hilbert spaces $\cH_K$ generated by a reproducing kernel $K$,
sampled at irregular high-dimensional points $X=(x_k)_{k \in \N}$, and for specific 
sequences $(L_k)_{k \in \N}$ in $\cH_K$.

\medskip
To be more precise, we introduce the notion of {\em determining} sequences of sampling points $X$
that are essentially linked with the reproducing kernel $K$ of the RKHS $\cH_K$. This then leads to
a closed linear subspace $\cH_{K,X}$ of $\cH_K$, in which all functions $f \in \cH_{K,X}$ can be represented
by a sampling formula as in~\eqref{eq:st}, and where $(L_k)_{k \in \N}$ is a Riesz basis of $\cH_{K,X}$.
For a general discussion on determining sets for functions, cf.~\cite{Zalcman1985}.
 
\medskip
Our approach is related to the work~\cite{Nashed1991} by Nashed and Walter, 
who explored sampling expressions in RKHS that are closed subspaces of specific Sobolev spaces.

\medskip
The outline of this paper is as follows.
We first explain key features on RKHS in Section~\ref{sec:rkhs}, 
before we recall only a few standard results from Riesz theory 
for the finite-dimensional case (in Section~\ref{sec:riesz}).
On these grounds, we develop sampling formulas 
for positive definite kernels $K$ (in Section~\ref{sec:sampling}),
that can be expressed through the solution of an infinite linear equation system 
in the sequence space $\ell^2$. 
Finally, we extend our results to tensor products of positive definite kernels 
(in Section~\ref{sec:tp:sampling}). The latter is highly relevant for efficient
tensor product approximation in high-dimensional data analysis. 

%%%%%%%%%%%%%%%%%%%%%%%%%%%%%%%%%%%
% R E P R O D U C I N G   K E R N E L   H I L B E R T   S P A C E S
%%%%%%%%%%%%%%%%%%%%%%%%%%%%%%%%%%%
%--------------------------------------------------------------------
\section{Reproducing Kernel Hilbert Spaces}
\label{sec:rkhs}
%--------------------------------------------------------------------
Kernels are popular tools for multivariate approximation.
To explain basic features of kernel-based approximation,
we introduce positive definite functions and their associated
reproducing kernel Hilbert spaces
(cf.~\cite{Buhmann2003,Iske2018,Wendland2005} for details).

\begin{definition}
\label{bas:def:pdf}
A continuous and symmetric function $K : \R^d \times \R^d \longrightarrow \R$ 
is said to be a {\em positive definite kernel} on $\R^d$, $K \in \pd (\R^d)$, 
if for any finite set of pairwise distinct points $X = \{x_1,\ldots,x_n\} \subset \R^d$, 
$n \in \N$, the matrix
\begin{equation*}
%\label{bas:equ:pdm}
    A_{K,X} = (K(x_k,x_j))_{1 \leq j,k \leq n} \in \R^{n \times n}
\end{equation*}
is symmetric and positive definite.
\end{definition}

Positive definite kernels $K \in \pd (\R^d)$ are often required to be {\em translation invariant}, 
i.e.,  $K$ is assumed to have the form
\begin{equation*}
%\label{bas:equ:tri}
    K(x,y) = \Phi(x - y) 
    \qquad \mbox{ for } x,y \in \R^d
\end{equation*}
for an even function $\Phi : \R^d \longrightarrow \R$.
Popular examples for translation invariant kernels $K \in \pd(\R^d)$ include the {\em Gaussian} 
$
  K(x,y) = \Phi(x-y) = \exp(- \| x-y \|_2^2),
$
and the {\em inverse multiquadric}
$
  K(x,y) = \Phi(x-y) = ( 1 + \| x - y \|_2^2 )^{-1/2}.
$

\medskip
Now let us explain a few key features of kernel-based approximation. 
In the following discussion, it will be convenient to let $K_x : \R^d \longrightarrow \R$, for $x \in \R^d$,  
be defined as
$$
    K_x(y) := K(x,y)
    \qquad \mbox{ for } x,y \in \R^d.
$$
Then, according to the seminal work of Aronszajn~\cite{Aronszajn1950},
the {\em reproducing kernel Hilbert space} (RKHS) $\cH_K$ associated with $K \in \pd(\R^d)$ is the closure
$$
    \cH_K := \overline{ {\rm span} \left\{ K_x \st x \in \R^d \right\} }
$$
with respect to the inner product $(\cdot,\cdot)_K \equiv (\cdot,\cdot)_{\cH_K}$ satisfying
$$
    (K_x,K_y)_K = K(x,y)
    \qquad \mbox{ for all } x,y \in \R^d,
$$
whereby we have
\begin{equation*}
    \left\| \sum_{j=1}^n c_j K_{x_j} \right\|_K^2 = 
%    \left( \sum_{j=1}^n c_j K_{x_j} , \sum_{k=1}^n c_k K_{x_k} \right)_K 
%    \\ &=& 
    \sum_{j,k=1}^n c_j K(x_j,x_k) c_k,
    =
    c^T A_{K,X} c,
\end{equation*}
for all $X = \{x_1,\ldots,x_n\} \subset \R^d$ and $c=(c_1,\ldots,c_n)^\top \in \R^n$.

\medskip
Therefore, $K \in \pd(\R^d)$ is the unique reproducing kernel of $\cH_K$, as characterized 
by the {\em reproduction property}
\begin{equation}
\label{eq:repro}
  f(x) = (f , K_x)_K
  \qquad \mbox{ for all } f \in \cH_K \mbox{ and all } x \in \R^d.
\end{equation}

\medskip
In this work, we provide sampling formulas for functions $f \in \cH_K$.
To this end, we fix a countable set of sampling points $X = (x_k)_{k \in \N}$ in $\R^d$,
which yields the closed subset
$$
    \cH_{K,X} := \overline{ {\rm span} \left\{ K_x \st x \in X \right\} } \subset \cH_K
$$
of the RKHS $\cH_K$.
On given samples $(f(x_k))_{k \in \N}$ taken from $f \in \cH_{K,X}$, 
we wish to prove sampling formulas of the form~\eqref{eq:st}
with specific basis functions $L_k \in \cH_{K,X}$, for $k \in \N$.
We remark at this point that for the representation in~\eqref{eq:st} 
to hold for any $f \in \cH_{K,X}$, we require 
conditions on the sampling points $X = (x_k)_{k \in \N}$. 
We will detail on this later in Section~\ref{sec:sampling}.

%--------------------------------------------------------------------
\section{Riesz Bases, their Duals, and Riesz Stability}
\label{sec:riesz}
%--------------------------------------------------------------------
Before we prove sampling formulas for functions $f \in \cH_{K,X}$ on {\em infinite} sets of sample points $X = (x_k)_{k \in \N}$,
let us first discuss the case of {\em finitely} many samples. 
To this end, we first recall a few standard results from Riesz theory, which we adapt to reproducing kernel Hilbert spaces.
This is partly for the purpose of preparation, and partly for the sake of readability.

\medskip
To introduce {\em Riesz bases} for finite-dimensional linear subspaces of the RKHS $\cH_{K}$,
let us fix a {\em finite} set of pairwise distinct sampling points $X = \{x_1,\ldots,x_n\}$.
This yields a finite set $\cR_X := \{ K_{x_j} \, | \, 1 \leq j \leq n \}$ of linearly independent functions, 
whose span $\cS_X := {\rm span} \{ K_{x_j} \, | \, 1 \leq j \leq n \}$ is a finite-dimensional linear subspace
of $\cH_{K}$.
Recall that the functions in $\cR_X$ are a {\em Riesz basis} of $\cS_X$.

\medskip
In fact, due to the Courant-Fischer theorem from linear algebra, the {\em Riesz stability estimate}
\begin{equation*}
  \lambda_{A_{K,X}} \| c \|_2^2 \leq 
  \left\| \sum_{j=1}^n c_j K(\cdot,x_j) \right\|_K^2 \leq
  \Lambda_{A_{K,X}} \| c \|_2^2
\end{equation*}
holds for all $c = (c_1,\ldots,c_n)^T \in \R^n$, and the {\em Riesz constants} are given by the 
smallest eigenvalue $\lambda_{A_{K,X}}$ and the largest eigenvalue $\Lambda_{A_{K,X}}$ 
of the positive definite kernel matrix $A_{K,X}$.

\medskip
Now we introduce the (unique) dual Riesz basis $\tilde{\cR}_X$ of $\cR_X$,
where it is straight forward to identify the Lagrange basis of $\cS_X$ as dual to $\cR_X$.
We can explain this as follows: 
First recall that the Lagrange basis $\cL_X = \{\ell_1,\ldots,\ell_n\}$ of $\cS_X$ 
can be characterized by $\ell_j(x_k) = \delta_{jk}$ for $1 \leq j,k \leq n$.
Moreover, the Lagrange basis functions in $\cL_X$ are uniquely determined by the solution of the linear system 
\begin{equation}
\label{pdf:lag:equ:lag}
  A_{K,X} \cdot \ell(x) = R_X(x)
  \qquad \mbox{ for } x \in \R^d,
\end{equation}
where
\begin{eqnarray*}
    \ell(x) &=& (\ell_1(x),\ldots,\ell_n(x))^T \in \R^n
\\ %    \mbox{ and }
    R_X(x) &=& (K_{x_1}(x),\ldots,K_{x_n}(x))^T \in \R^n.
\end{eqnarray*}

Hence, the (unique) interpolant $s \in \cS_X$ to $f$ on $X$ has the {\em Lagrange representation}
\begin{equation*}
%\label{pdf:lag:equ:lgf1}
  s(x) = \langle f_X , \ell(x) \rangle,
\end{equation*}
where we let $f_X = (f(x_1),\ldots,f(x_n))^T \in \R^n$, and where
$\langle \cdot, \cdot \rangle$ denotes the inner product on the Euclidean space $\R^n$.

\medskip
We can summarize our discussion as follows.

\begin{theorem}
\label{pdf:sta:rie:drb}
For a finite point set $X = \{x_1,\ldots,x_n\}$, $n \in \N$,
the unique dual Riesz basis of $\cR_X = \{ K_{x_j} \, | \, x_j \in X \}$
is given by the Lagrange basis of $\cS_X$, i.e., $\tilde{\cR}_X = \cL_X$.
In particular,
\begin{itemize}
\item[(a)]
the orthonormality relation
\begin{equation}
%\label{eq:biortho}
\label{pdf:sta:rie:bio}
  ( K_{x_j}, \ell_k )_K = \delta_{jk}
\end{equation}
holds for all $1 \leq j,k \leq n$. 
\item[(b)]
for any $f_X \in \R^n$, we have the stability estimate
\begin{equation}
\label{pdf:sta:equ:sas1}
  \Lambda^{-1}_{A_{K,X}} \| f_X \|_2^2 \leq 
  \left\| \sum_{j=1}^n f(x_j) \ell_j \right\|_K^2 \leq 
  \lambda^{-1}_{A_{K,X}} \| f_X \|_2^2.
\end{equation}
\item[(c)]
for any $s \in \cS_X$, we have the stability estimate
\begin{equation}
\label{pdf:sta:equ:sas2}
  \lambda_{A_{K,X}} \| s \|_K^2 \leq \| s_X \|_2^2 \leq \Lambda_{A_{K,X}} \| s \|_K^2.
\end{equation}
\end{itemize}
\end{theorem}

\begin{proof}
The orthonormality relation in~(\ref{pdf:sta:rie:bio}) follows 
immediately from the reproduction property in~\eqref{eq:repro}, whereby
$$
  (K_{x_j}, \ell_k)_K = \ell_k(x_j) = \delta_{jk}
  \qquad \mbox{ for all } 1 \leq j,k \leq n.
$$
Due to the fundamental duality property from Riesz theory~\cite{Christensen2016},
the Lagrange basis $\cL_X$ is the unique dual Riesz basis of $\cR_X$.
Moreover, for any $f_X \in \R^n$, we have the representation
$$
  \left\| \sum_{j=1}^n f(x_j) \ell_j \right\|_K^2 = 
  \| f_X \|^2_{A_{K,X}^{-1}} = f^T_X A^{-1}_{K,X} f_X,
$$
and by the Courant-Fischer theorem, the {\em Rayleigh estimates}
$$
  \lambda_{A^{-1}_{K,X}} \| f_X \|_2^2 
  \leq f^T_X A^{-1}_{K,X} f_X \leq
  \Lambda_{A^{-1}_{K,X}} \| f_X \|_2^2
$$
hold.
This implies the stability estimate in~(\ref{pdf:sta:equ:sas1}), where
$$
   \Lambda^{-1}_{A_{K,X}} = \lambda_{A^{-1}_{K,X}}
   \quad \mbox{ and } \quad
   \lambda^{-1}_{A_{K,X}} = \Lambda_{A^{-1}_{K,X}}.
$$   
Letting $f = s \in \cS_X$ in~(\ref{pdf:sta:equ:sas1}), we finally get
$$
  \Lambda^{-1}_{A_{K,X}} \| s_X \|_2^2 \leq 
  \| s \|_K^2 = \left\| \sum_{j=1}^n s(x_j) \ell_j \right\|_K^2 \leq
  \lambda^{-1}_{A_{K,X}} \| s_X \|_2^2
$$
for all $s \in \cS_X$, so that the stability estimate in~(\ref{pdf:sta:equ:sas2}) holds.
\end{proof}

Let us recall another important result
from the Riesz duality relation between 
$\cR_X = \{ K_{x_j} \, | \, x_j \in X \}$ 
and
$\tilde{\cR}_X = \cL_X$.

\begin{corollary}
\label{pdf:sta:cor:dar}
For any $f \in \cS_X$, we have the representations
\begin{equation}
\label{pdf:sta:equ:dar}
  f = \sum_{j=1}^n (f , K_{x_j})_K \, \ell_j
    = \sum_{j=1}^n (f , \ell_j)_K \, K_{x_j}.
\end{equation}
\qed
\end{corollary}

We can also verify the representations in~(\ref{pdf:sta:equ:dar}) more directly.
Indeed, for any $f \in \cS_X$, we have the two representations 
\begin{equation*}
%\label{pdf:sta:equ:rep}
    f = \sum_{j=1}^n f(x_j) \ell_j = \sum_{j=1}^n c_j K_{x_j}, 
\end{equation*}
one with respect to the Lagrange basis $\cL_X$, the other with respect to the kernel basis $\cR_X$.
Now, on the one hand, we get
\begin{equation}
\label{one:hand}
    c_j = \langle e_j , c \rangle = e_j^T A_{K,X}^{-1} f_X = f_X^T A_{K,X}^{-1} e_j = (f , \ell_j)_K
\end{equation}
from~\eqref{pdf:lag:equ:lag}.
On the other hand, we have $(f , K_{x_j})_K = f(x_j)$
by the reproduction property~\eqref{eq:repro} of $K \in \pd(\R^d)$, 
for $1 \leq j \leq n$.

%--------------------------------------------------------------------
\section{Sampling Formulas for RKHS}
\label{sec:sampling}
%--------------------------------------------------------------------
In this section, we develop sampling formulas of the form~\eqref{eq:st}
for {\em infinite} sets of sampling points $X = (x_n)_{n \in \N}$.
To this end, we require determining sequences of sampling points in $X = (x_n)_{n \in \N}$, 
according to the following definition.

\smallskip
\begin{definition}
\label{def:suitable}
Let $K \in \pd(\R^d)$ be a positive definite kernel.
Then, a countable sequence $X = ( x_n )_{n\in\N}$ 
of pairwise distinct sampling points in $\R^d$
said to be {\em determining} in the RKHS $\cH_K$,
if the infinite kernel matrix 
\begin{equation}
\label{inf:kernel:matrix}
    A_{K,X} = \left( K(x_j,x_k) \right)_{j,k=1}^\infty
\end{equation}
is bijective and bounded on $\ell_2(\mathbb{N})$, i.e.,
$A_{K,X} \in L (\ell_2(\mathbb{N}),\ell_2(\mathbb{N}))$.
\end{definition}

We remark that the property for $X = ( x_n )_{n\in\N}$ to be a determining sequence in $\cH_K$, 
according to Definition~\ref{def:suitable},
is equivalent for the sequence $(K_{x_n})_{n \in \N}$ to be a Riesz basis in $\cH_{K,X}$.
We can reformulate this standard result from Riesz theory as follows.

\begin{lemma}
\label{lemma:charact}
For $K \in \pd(\R^d)$, let $X = ( x_n )_{n \in \N}$ be a determining sequence in $\cH_K$.
Then, any $f \in \cH_{K,X}$ is characterised by the representation
$$
    f = \sum_{k \in \N} c_k K_{x_k}
$$
for some coefficient vector $c = (c_k)_{k \in \N} \in \elltwo$. 
\end{lemma}

\begin{proof}
We follow along the lines of~\cite[Theorem 3.6.6]{Christensen2016}:
$(K_{x_n})_{n \in \N}$ is a Riesz basis of $\cH_{K,X}$, 
since the corresponding Gram matrix $G_{K,X} := \left( ( K_{x_j}, K_{x_k} )_K \right)_{j,k \in \N}$ 
contains the entries $K(x_j,x_k) = ( K_{x_j},K_{x_k} )_K$, i.e., $G_{K,X} = A_{K,X}$.
\end{proof}

The result of Lemma~\ref{lemma:charact} is our starting point for
the development of sampling formulas in the RKHS $\cH_{K,X}$.
To this end, we will show that the sampling functions $L_k \in \cH_{K,X}$ 
in~\eqref{eq:st} are those from the dual Riesz basis of $(K_{x_n})_{n \in \N}$. 

\medskip
The following result shows that there is a one-to-one relation between 
the elements in $\cH_{K,X}$ (viewed as equivalence classes of Cauchy sequences)
and the pointwise representation of functions $f$ in ${\rm span} \{ K_x \, | \, x \in X \}$
with coefficients in $\elltwo$.

\begin{lemma}
\label{lemma:auchtoll}
For $K \in \pd(\R^d)$, let $X = ( x_n )_{n \in \N}$ be a determining sequence in $\cH_K$.
Moreover, for $c = (c_k)_{k \in \N} \in \elltwo$, let $f : \R^d \longrightarrow \R$ be defined as
$$
    f(x) = \sum_{k \in \N} c_k K_{x_k}(x)
    \qquad \mbox{ for } x \in \R^d.
$$
Then, there is one unique $\tilde{f} \in \cH_{K,X}$ satisfying
$$
     ( \tilde{f}, K_x )_K = f(x)
     \qquad \mbox{ for all } x\in\R^d.
$$  
\end{lemma}

\begin{proof}
Let $\tilde{f} \in \cH_{K,X}$ be defined as
$$
     \tilde{f} := \sum_{k \in \N} c_k K_{x_k}.
$$
Moreover, for $n \in \N$, let $\tilde{f_n}$ be the $n$-th partial sum of $\tilde{f}$, i.e.,
$$ 
    \tilde{f_n} = \sum_{k=1}^n c_k K_{x_k} \in \cH_{K,X}.
$$ 
Then, we obtain
\begin{eqnarray*}
( \tilde{f},K_x )_K
 &=& \lim_{n \to \infty} ( \tilde{f_n} , K_x )_K 
=
    \lim_{n \to \infty} \left( \sum_{k=1}^n c_k K_{x_k} , K_x \right)_K 
    \\ &=& 
    \lim_{n \to \infty} \sum_{k=1}^n c_k ( K_{x_k} , K_x )_K 
=
    \lim_{n \to \infty} \sum_{k=1}^n c_k K_{x_k}(x)
    \\ & = & f(x)
\end{eqnarray*}
by the reproduction property $K_{x_k}(x) = ( K_{x_k} , K_x )_K$.
\end{proof}

Now we are in a position, where we can formulate one sampling formula in $\cH_{K,X}$.
To this end, we use the invertibility of the infinite kernel matrix $A_{K,X}$ in~\eqref{inf:kernel:matrix}.
Note that for any 
$$
    f = \sum_{k \in \N} c_k K_{x_k} \in \cH_{K,X},
$$
the function values of $f_X = (f(x_n))_{n \in \N}$ at the sampling points $X = (x_n)_{n \in \N}$ 
determine $f$ uniquely, since 
$f_X= (f(x_n))_{n \in \N} = A_{K,X} c$, whereby $c = A_{K,X}^{-1}f_X$,
like in the finite dimensional case~\eqref{one:hand}.

\begin{theorem}
\label{theo:samplingexpressionf}
For $K \in \pd(\R^d)$, let $X = ( x_n )_{n \in \N}$ be a determining sequence in $\cH_K$.
Then, any $f \in \cH_{K,X}$ can be expressed by the sampling formula
\begin{equation}
\label{eq:samplingexpressionf}
        f = \sum_{k \in \N} f(x_k) L_k,
\end{equation}
where
\begin{equation*}
%\label{eq:samplingfunctionexpression}
        L_k = \sum_{n \in \N} (A_{K,X}^{-1}e_k)_n K_{x_n} \in \cH_{K,X}
\qquad \mbox{ for } k \in \N,
\end{equation*}
and where $e_k$ is the $k$-th unit vector in $\elltwo$.
\end{theorem}

\begin{proof}
We can express $f \in \cH_{K,X}$ pointwise by
$$
    f(x)= \sum_{k \in \N} c_k K_{x_k}(x) = \langle c, R_X(x) \rangle_{\ell_2},
$$
for specific coefficients $c = (c_k)_{k \in \N} \in \ell_2$, and where we let 
$$
    R_X(x) := (K_{x_n}(x))_{n \in \N} \in \ell_2,
$$
cf.~\eqref{pdf:lag:equ:lag}.
Therefore, we have
$$
    f(x) = \langle A_{K,X}^{-1} f_X, R_X(x) \rangle_{\ell_2} = \langle f_X, A_{K,X}^{-1} R_X(x) \rangle_{\ell_2},
$$
since $A_{K,X}^{-1}$ is self-adjoint. 
Now we let $L(x):=A_{K,X}^{-1} R_X(x)$, and so
$$
    L_k(x)=\langle L(x),e_k \rangle_{\ell_2} = \langle R_X(x) , A_{K,X}^{-1} e_k \rangle_{\ell_2}.
$$
By Lemma~\ref{lemma:charact}, we have $L_k \in \cH_{K,X}$ for all $k \in \N$. 

\medskip
Finally, $(K_{x_j})_{j \in \N}$ and $(L_k)_{k \in \N}$ are biortho\-gonal by
\begin{eqnarray*}
  ( L_k , K_{x_j} )_K &=& L_k (x_j) = \langle L(x_j) , e_k \rangle_{\ell^2}
  \\&=&
 \langle A_{K,X}^{-1} R_X(x_j) , e_k \rangle_{\ell^2} 
 = \langle e_j , e_k \rangle_{\ell^2} 
 \\ &=& 
 \delta_{jk}.
\end{eqnarray*}
\end{proof}

\begin{corollary}
\label{col:rieszbases}
For $K \in \pd(\R^d)$, let $X = ( x_n )_{n \in \N}$ be a determining sequence in $\cH_K$.
Then, $(K_{x_j})_{j \in \N}$ and $(L_k)_{k \in \N}$ are dual Riesz bases in $\cH_{K,X}$.
\end{corollary}

\begin{proof}
The functions $(K_{x_j})_{j \in \N}$ are a Riesz basis. 
This is because the Gram matrix $G_{K,X} := \left( ( K_{x_j}, K_{x_k} )_K \right)_{j,k \in \N}$  
is bounded and invertible on $\elltwo$. 
By straight forward calculations we get
\begin{equation*}
    f = \sum_{k \in \N} ( f , L_k )_K K_{x_k} =\sum_{k \in \N} ( f , K_{x_k} )_K L_k,
\end{equation*}
and, moreover, $\left( ( L_j ,L_k )_K \right)_{j,k \in \N} = A_{K,X}^{-1}$, i.e.,
$(L_k)_{k \in \N}$ is the dual Riesz basis of $(K_{x_j})_{j \in \N}$ in $\cH_{K,X}$.
\end{proof}

%--------------------------------------------------------------------
\section{Sampling Formulas for Tensor Product RKHS}
\label{sec:tp:sampling}
%--------------------------------------------------------------------
Now we extend the sampling formula~\eqref{eq:samplingexpressionf} of Theorem~\ref{theo:samplingexpressionf} to 
{\em tensor products} of RKHS (TP-RKHS), where 
we rely on previous findings on tensor products of frames~\cite{Bourouihiya2008} and on tensor product kernels~\cite{AEI2025}.

\medskip
Let us first give a definition for products of Hilbert spaces.

\begin{definition}
\label{defi:tp:hilbert}
The {\em Hilbert tensor product} of two real Hilbert spaces $\mathcal{H}_1,\mathcal{H}_2$ 
is the unique (up to isomorphisms) Hilbert space $\mathcal{H}:=\mathcal{H}_1\bigotimes \mathcal{H}_2$ 
for which there exists a bilinear mapping $\varphi:\mathcal{H}_1\times\mathcal{H}_2 \longrightarrow \mathcal{H}$
satisfying
\begin{equation}
\label{eq:eigenschaftbilinear}
( \varphi(f_1,f_2),\varphi(g_1,g_2) )_{\mathcal{H}}= ( f_1,g_1)_{\mathcal{H}_1} \cdot ( f_2,g_2 )_{\mathcal{H}_2},
\end{equation}
for all $(f_1,f_2),(g_1,g_2)\in\mathcal{H}_1\times\mathcal{H}_2$.
\end{definition}

In the special case of two reproducing kernel Hilbert spaces $\cH_{K_1}$ and $\cH_{K_2}$ 
from kernels $K_1 \in \pd(\R^{d_1})$ and $K_2 \in \pd(\R^{d_2})$,
their Hilbert tensor product $\cH_{K_1} \bigotimes \cH_{K_2}$ is given by the 
RKHS generated by $K : \R^{d_1+d_2} \longrightarrow \R$, where
\begin{equation}
\label{eq:tp:kernel}
    K(\mathbf{x},\mathbf{y})=K_1(x_1,x_2) \cdot K_2(y_1,y_2)
\end{equation}
for $\mathbf{x} = (x_1,y_1)$ and $\mathbf{y} = (x_2,y_2) \in \R^{d_1+d_2}$.

\medskip
A quick way to see that $\cH_{K_1} \bigotimes \cH_{K_2}$ must necessarily be generated by the
tensor product kernel $K$ in~\eqref{eq:tp:kernel}, 
is to let $\varphi(f,g)(x) := f(x)g(x)$ (in Definition~\ref{defi:tp:hilbert}),
whereby
\begin{eqnarray*}
   ( K(\cdot,\mathbf{x}), K(\cdot, \mathbf{y}) )_K 
   & = &
    K( \mathbf{x}, \mathbf{y}) =K_1(x_1,x_2) \cdot K_2(y_1,y_2)
    \\ &=& 
    ( K_1(\cdot,x_1),K_1(\cdot,x_2) )_{K_1} \cdot ( K_2(\cdot,y_1), K_2(\cdot,y_2) )_{K_2}.
\end{eqnarray*}
Therefore, \eqref{eq:eigenschaftbilinear} holds for all basis elements $K( \cdot, \mathbf{x} )$ 
of the RKHS $\cH_K$, and so it can be extended to all elements in $\cH_K$. 

\medskip
As recently proven in~\cite{AEI2025}, the tensor product $K$ between two positive definite 
kernels $K_1 \in \pd(\R^{d_1})$ and $K_2 \in \pd(\R^{d_2})$,
as in~\eqref{eq:tp:kernel}, is positive definite, i.e., $K \in \pd(\R^{d_1+d_2})$.

\medskip
Moreover, for Hilbert tensor products $\mathcal{H}=\mathcal{H}_1\bigotimes\mathcal{H}_2$,
their Riesz bases are given by the tensor products between the Riesz bases of the Hilbert spaces $\mathcal{H}_1$ and $\mathcal{H}_2$,
which is due to~\cite{Bourouihiya2008}.

\begin{lemma}
\label{lemma:rieszbasistensorproduct}
Let $\mathcal{H}_1,\mathcal{H}_2$ be Hilbert spaces. 
If $(f_{1,j})_{j \in \N}$ is a Riesz basis for $\mathcal{H}_1$, 
and $(f_{2,k})_{k \in \N}$ is a Riesz basis for $\mathcal{H}_2$,
then $(\varphi(f_{1,j},f_{2,k}))_{(j,k) \in \N \times\N}$ is a
Riesz basis for $\mathcal{H}_1\bigotimes\mathcal{H}_2$.
\qed
\end{lemma}

\newpage

%\medskip
Now we can extend our results to tensor products of RKHS.

\begin{theorem}
For $K_1 \in \pd(\R^{d_1})$, let $X = ( x_n )_{n \in \N}$ be a determining sequence in $\cH_{K_1}$,
and for $K_2 \in \pd(\R^{d_2})$, let $Y = ( y_m )_{m \in \N}$ be a determining sequence in $\cH_{K_2}$.
Then, the tensor product kernel
$$
    K(\mathbf{x},\mathbf{y}) := K_1(x_1,y_1)\cdot K_2(x_2,y_2)
    \qquad \mbox{ for } \mathbf{x}, \mathbf{y} \in \R^{d_1+d_2}
$$
is positive definite on $\R^{d_1+d_2}$, i.e., $K \in \pd(\R^{d_1+d_2})$
and $X\times Y$ is a determining sequence in $\cH_K$.
Moreover,
$$
    (K_1(\cdot,x_j)\cdot K_2(\cdot,y_k))_{j,k\in\N^2}
    \quad \mbox{ and } \quad
    (L_{1,j}\cdot L_{2,k})_{j,k\in\N^2}
$$
are dual Riesz bases in $\cH_{K, X \times Y}$.
\end{theorem}

\begin{proof}
The positive definiteness of $K$ is proven in \cite{AEI2025}, 
and the Hilbert product space of the native spaces of $K_1$ and $K_2$ yields the native space of $K$. 
Since $X$ and $Y$ are determining for $\cH_{K_1}$ and $\cH_{K_2}$, respectively,
we can consider their dual Riesz bases as stated in Corollary~\ref{col:rieszbases}. 
Then, we can compute Riesz bases for $\cH_{K,X \times Y}$ by the Hilbert tensor product of the 
Riesz bases for $\cH_{K_1,X}$ and $\cH_{K_2,Y}$. 
By letting $\varphi(f,g)(x)=f(x)g(x)$ in Definition~\ref{defi:tp:hilbert}, we obtain the stated result. 
\end{proof}

\section{Conclusion and Future Work}
We have introduced determining sequences of irregular samples
to obtain sampling formulas for high-dimensional functions from
reproducing kernel Hilbert spaces (RKHS). 
Our construction relies on standard tools from sampling theory,
where, in particular, we combine approximation methods with 
basic results from Riesz theory. 
We have extended our results to sampling formulas for tensor RKHS,
where we rely on our recent results on tensor product kernels~\cite{AEI2025}.
This provides a concept for efficient representations of multivariate functions  
in high dimensions.
Yet, it remains to elaborate suitable characterisations for the construction of determining samples.

%%%%%%%%%%%%%%%%%%%%%%%%%%%%%%%%%%%
% R E F E R E N C E S
%%%%%%%%%%%%%%%%%%%%%%%%%%%%%%%%%%%
% Non-BibTeX users please use


\begin{thebibliography}{11}
%
%\providecommand{\url}[1]{\texttt{#1}}
%\expandafter\ifx\csname urlstyle\endcsname\relax
%  \providecommand{\doi}[1]{doi: #1}\else
%  \providecommand{\doi}{doi: \begingroup \urlstyle{rm}\Url}\fi
%
%
\bibitem{AEI2025}
K.~Albrecht, J.~Entzian, A.~Iske:
\emph{Product kernels are efficient and flexible tools for high-dimensional scattered data interpolation}.
{\sl Adv.\ Comput.\ Math.}~{\bf 51}, 14 (2025), 
https://doi.org/10.1007/s10444-025-10226-y.
%
\bibitem{AEI2023}
K.~Albrecht, J.~Entzian, A.~Iske:
\emph{Anisotropic kernels for particle flow simulation}.
In: {\sl Modeling, Simulation and Optimization of Fluid Dynamic Applications}. 
Springer Nature Switzerland, 2023, 57--76.
%
\bibitem{AI2025}
K.~Albrecht and A.~Iske:
\emph{On the convergence of generalized kernel-based interpolation by greedy data selection algorithms}. 
{\sl BIT Numer.\ Math.}~{\bf 65}, 5 (2025). https://doi.org/10.1007/s10543-024-01048-3.
%
\bibitem{Aronszajn1950}
N.~Aronszajn:
\emph{Theory of reproducing kernels}.
{\sl Transactions of the American Mathematical Society}~{\bf 68} (3), 1950, 337--404.
%
\bibitem{Bartolucci2023}
F.~Bartolucci, E.~De Vito, L.~Rosasco, and S.~Vigogna: 
\emph{Understanding neural networks with reproducing kernel Banach spaces}. 
{\sl Appl.\ Comput.\ Harmon.\ Anal.}~{\bf 62}, 2023, 194--236.
%
\bibitem{BIP2001}
J.~Behrens, A.~Iske, and S.~P\"ohn:
\emph{Effective node adaption for grid-free semi-Lagrangian advection}.
{\sl Discrete Modelling and Discrete Algorithms in Continuum Mechanics}.
Logos, Berlin, 2001, 110--119.
%
\bibitem{Berlinet2004}
A.~Berlinet and C.~Thomas-Agnan: 
{\sl Reproducing Kernel Hilbert Spaces in Probability and Statistics}. 
Kluwer, Boston, 2004.
%
\bibitem{Bourouihiya2008}
A.~Bourouihiya:
\emph{The Tensor Product of Frames}.
{\sl Sampling Theory in Signal and Image Processing}~{\bf 7} (1), 2008, 65--76.
%
\bibitem{Buhmann2003}
M.D.~Buhmann:
{\sl Radial Basis Functions}.
Cambridge University Press, Cambridge, UK, 2003.
%
\bibitem{Christensen2016}
O.~Christensen:
{\sl An Introduction to Frames and Riesz Bases}. 
Second edition. Birkh\"auser, 2016.
%
\bibitem{Cucker2007}
F.~Cucker and D.X.~Zhou:
{\sl Learning Theory: An Approximation Theory Viewpoint}. 
Cambridge University Press, 2007.
%
\bibitem{EIT2023}
S.~Eckstein, A.~Iske, and M.~Trabs:
\emph{Dimensionality reduction and Wasserstein stability for kernel regression}.
{\sl Journal of Machine Learning Research}~{\bf 24} (334), 2023, 1--35.
%
\bibitem{Fukumizu2009}
K.~Fukumizu, F.~Bach, and M.~Jordan:
\emph{Kernel dimension reduction in regression}.
{\sl Annals of Statistics}~{\bf 37}, 2009, 1871--1905.
%
\bibitem{Guillemard2017}
M.~Guillemard and A.~Iske.
\emph{Interactions between kernels, frames, and persistent homology}.
In: {\em Recent Applications of Harmonic Analysis to Function Spaces,
Differential Equations, and Data Science. Volume 2: Novel Methods in Harmonic Analysis},  
Birkh\"auser, Basel, 2017, 861--888.
%
\bibitem{Iske2018}
A.~Iske:
{\sl Approximation Theory and Algorithms for Data Analysis}.
Texts in Applied Mathematics 68, Springer, Cham, 2018.
%
\bibitem{Micchelli2006}
C.A.~Micchelli, Y.~Xu, H.~Zhang: 
\emph{Universal kernels}. 
{\sl Journal of Machine Learning Research}~{\bf 7}, 2006, 2651--2667.
%
\bibitem{Nashed1991}
M.Z.~Nashed and G.G.~Walter:
\emph{General sampling theorems for functions in reproducing kernel Hilbert spaces}. 
{\sl Math.\ Control Signal Systems}~{\bf 4}, 1991.
%
\bibitem{Schaback2006}
R.~Schaback and H.~Wendland:
\emph{Kernel techniques: From machine learning to meshless methods}.
{\sl Acta Numerica}~{\bf 15}, 2006, 543--639.
%
\bibitem{Schoelkopf2002}
B.~Sch\"olkopf and A.J.~Smola: 
{\sl Learning with Kernels}. 
MIT Press, Cambridge, 2002.
%
\bibitem{Shannon1949}
C.E.~Shannon: 
\emph{Communication in the presence of noise}. 
{\sl Proc.\ IRE}~{\bf 37}, 1949, 10--21.
%
\bibitem{Steinwart2008}
I.~Steinwart and A.~Christmann: 
{\sl Support Vector Machines}. 
Springer, New York, 2008.
%
\bibitem{Unser2000}
M.~Unser:
\emph{Sampling -- 50 years after Shannon}.
{\sl Proceedings of the IEEE}~{\bf 88} (4), April 2000, 569--587.
%
\bibitem{Wendland2005}
H.~Wendland:
{\sl Scattered Data Approximation}.
Cambridge University Press, Cambridge, UK, 2005.
%
\bibitem{Zalcman1985}
L.A.~Zalcman:
\emph{Determining sets for functions and measures}. 
The ninth summer real analysis symposium (Louisville., 1985). 
{\sl Real Anal.\ Exchange}~{\bf 11}(1), 1985/86, 40--55.
%
\bibitem{Zhou2003}
D.-X.~Zhou:
\emph{Capacity of reproducing kernel spaces in learning theory}.
{\sl IEEE Transactions on Information Theory}~{\bf 49} (7), July~2003, 1743--1752.
%
\end{thebibliography}
\end{document}